\documentclass[letterpaper, 10 pt, conference]{ieeeconf}  % Comment this line out if you need a4paper

\IEEEoverridecommandlockouts                              % This command is only needed if 
\usepackage[pdftex]{graphicx}
\usepackage[inkscapeformat=png]{svg}
\usepackage{amsmath}
\usepackage{amsfonts}

\usepackage{amsthm}
\usepackage{xfrac}
\usepackage{hyperref}
\usepackage[english]{babel}
\usepackage{mathtools}
\usepackage{algorithmic}
\usepackage[caption=false, font=footnotesize]{subfig}
\usepackage{mwe}
\usepackage{cite}
\usepackage{bm}
\usepackage{wrapfig}

\newcommand{\ignore}[1]{}

 \def\U{\mathcal{U}}

  \def\D{\mathcal{D}}
  
\def\X{\mathcal{X}}

\def\dR{\mathbb{R}}

\def\eps{\varepsilon}

\theoremstyle{definition}

\newtheorem{problem}{Problem}
\theoremstyle{theorem}
\newtheorem{lemma}{Lemma}

\newtheorem{theorem}{Theorem}

\def\niceparagraph#1{\vspace{5pt} \noindent \textbf{#1}}

\renewcommand{\QED}{\QEDopen}

\title{\LARGE \bf
Impossibility of Self-Organized Aggregation without Computation
}

\author{Roy Steinberg$^{*}$ and Kiril Solovey$^{*}$% <-this % stops a space
\thanks{$^{*}$ Electrical \& Computer Engineering, Technion, Haifa, Israel.
        {\tt\small \{roysteinberg@campus.,kirilsol@\}technion.ac.il}} %
}

\begin{document}

\maketitle
\thispagestyle{empty}
\pagestyle{empty}

%%%%%%%%%%%%%%%%%%%%%%%%%%%%%%%%%%%%%%%%%%%%%%%%%%%%%%%%%%%%%%%%%%%%%%%%%%%%%%%%
\begin{abstract}
 %Research on swarm robotics seeks to develop low-cost systems that collectively accomplish complex tasks while relying on sismple and local control laws. 
 In their seminal work, Gauci et al. (2014) studied the fundamental task of aggregation, wherein multiple robots need to gather without an a priori agreed-upon meeting location, using minimal hardware. That paper considered differential-drive robots that are memoryless and unable to compute. Moreover, the robots cannot communicate with one another and are only equipped with a simple sensor that determines whether another robot is directly in front of them. Despite those severe limitations, Gauci et al.\ introduced a controller and proved mathematically that it aggregates a system of two robots for any initial state. Unfortunately, for larger systems, the same controller aggregates empirically in many cases but not all. Thus, the question of whether a controller exists that aggregates for any number of robots remains open. In this paper, we show that no such controller exists by investigating the geometric structure of controllers. In addition, we disprove the aggregation proof of the paper above for two robots and present an alternative controller alongside a simple and rigorous aggregation proof. 
\end{abstract}

%%%%%%%%%%%%%%%%%%%%%%%%%%%%%%%%%%%%%%%%%%%%%%%%%%%%%%%%%%%%%%%%%%%%%%%%%%%%%%%%

\section{Introduction}\label{sec:intro}
Diverse species, from ants to birds, fish, and mammals, exhibit swarm intelligence, wherein relatively simple mechanisms deployed by individuals lead to an emergent collective behavior. Taking inspiration from the natural world, research on swarm robotics seeks to develop low-cost robots that can collectively execute complex tasks only by relying on simple and local control laws and without explicit communication. This includes collective decision-making between robots and animals~\cite{halloy2007social}, robots building complex structures~\cite{werfel2014designing}, and particle assembly via global control~\cite{Blumenberg0B23}.

A fundamental emergent behavior is aggregation (also known as gathering), where the robots must reach sufficiently close to one another without a priori agreeing upon a predefined meeting location.  In recent years, various approaches have been introduced towards swarm aggregation, from genetic programming ~\cite{trianni2003evolving}, to settings with limited sensing abilities~\cite{barel2021probabilistic,dovrat2017gathering}, and algorithms relying on potential functions% controlling the interaction between the robots
~\cite{gazi2005swarm,cohen2005convergence}. Although some approaches guarantee aggregation, they usually require complex computational abilities and memory.

The distributed computation community has also considered aggregation (see, e.g.,~\cite{oasa1997robust,cieliebak2012distributed,flocchini2019distributed}) where solutions that guarantee aggregation for any number of robots have been developed by relying on geometric control laws (e.g., moving towards the center of gravity of the robots). However, the proposed methods rely on strong assumptions 
such as knowledge of the locations of all the robots and the absence of kinematic motion or collision avoidance constraints, making them unsuitable in practice. 

Aggregation under a far less restrictive set of assumptions has been studied in~\cite{gauci2014self}, where the authors considered a memoryless controller for a group of homogeneous differential-drive (DD) robots. Here, the robots can obtain information about other robots through a binary sensor indicating whether another robot is in front of them. %(i.e., along a single ray from the robot's front). 
Considering those constraints, the controller of a given robot is \emph{bimodal}, as it specifies a robot's action depending on whether there is another robot in front of it. In particular, a bimodal controller is a four-dimensional vector specifying the speed of each of the two robot wheels for every mode.

To find an aggregating bimodal controller, the authors of~\cite{gauci2014self} empirically obtained the controller $u_{prev}:=(-0.7,-1,1,-1)$ through an exhaustive search of the four-dimensional control space. This controller yields a backward circular motion of the robot when no other robot is in sight, and a rotation on the spot otherwise. The paper provided a theoretical proof for aggregating two robots using the controller.
However, no proof has been shown for a larger robot number $n >2$. Instead, simulations and hardware experiments using 'e-puck' DD robots~\cite{mondada2009puck} have demonstrated that aggregation occurs within the allotted time budget in most cases.

The idea  that  $u_{prev}$ leads to aggregation for any $n\geq 2$ has recently been challenged. The paper~\cite{daymude2021deadlock} has shown a counterexample consisting of initial robot locations in which no aggregation would occur for any $n\geq 4$. In particular, by placing pairs
%\footnote{This construction can be generalized to an odd number of robots $n\geq 5$ by substituting one of the pairs with a triplet of forward-facing robots.} 
of robots facing directly away from each other, deadlocking can occur (see Fig.~\ref{fig:deadlocking}), as each robot attempts to move backward but is blocked by its partner. This causes each pair to stagnate. Interestingly, it is demonstrated empirically in~\cite{daymude2021deadlock} that some levels of noise (with respect to sensor and motion models) can help in escaping deadlocks for $u_{prev}$, but too much noise can be detrimental for aggregation. Unfortunately, no theoretical analysis is provided. The question posed by the authors of~\cite{gauci2014self} whether there exists an aggregating controller for any $n\geq 2$ remains open.

%Loosely borrowing Gauci's terminology, the robots' motion can be split into six categories; straight forwards and backwards motion denoted as 'SF' and 'SB' respectively, circular forwards and backwards motion denoted as 'CF' and 'CB' respectively, a rotation on the spot denoted as 'RS', and no motion at all denoted as 'SS' ('Stand Still'). The controller used was bimodal, which dictates a backwards circular motion (CB) when another robot is not in its line of sight, and a stationary rotation (RS) when there is line of sight. An example of these controllers can be seen in Fig. \ref{fig:motion types example}.

\niceparagraph{Contribution.} In this paper, we answer the above question in the negative for a general number of robots, alongside strengthening the understanding of the case of two robots. Our contributions are as follows. 
\textbf{(i)} We identify an implicit (and unreasonable) assumption made in the aggregation proof for two robots using  $u_{prev}$ in~\cite{gauci2014self}, which deems the proof incomplete. 
\textbf{(ii)} We present an alternative controller alongside a simple and rigorous aggregation proof for two robots. 
\textbf{(iii)} We prove that no bimodal controller can achieve aggregation for \emph{all} $n>2$. %This is done by devising counterexamples for each bimodal controller type while exploiting the geometry of the induced motion. 
\textbf{(iv)} We mathematically prove that for some controller types, initial robot states that lead to nonaggregation do not have to be singular. That is, we show that when the initial states are chosen at random, the probability of nonaggregation is strictly larger than zero.
%The above counterexamples require a special arrangement of the robots which induces a zero-measure volume in the joint robot state space. That is, if the initial positions of the robots are randomly sampled, then the probability of sampling from those particular ``infeasible'' configurations is equal to zero. Nevertheless, we show that for some types of bimodal controllers our results in (iii) can be generalized to positive volume configurations.
\textbf{(v)} Finally, we empirically demonstrate that our $2$-robot controller outperforms the controller $u_{prev}$, and test the effect of noise and slippage on nonaggregation.

%\niceparagraph{Organization.} In Sec.~\ref{sec:pre} we provide mathematical preliminaries and formally define the problem of aggregation. 

%\kiril{Mention anything experimental that is interesting.}

%we propose a slightly modified controller for said robots, which causes aggregation in all cases for 2 robots. For a general amount of robots, we prove that aggregation using no computation and memory is impossible for the introduced controller, and look at options with limited computation and memory instead.

%%%%%%%%%%%%%%%%%%%%%%%%%%%%%%%%%%%%%%%%%%%%%%%%%%%%%%%%%%%%%%%%%%%%%%%%%%%%%%%%%%%%%%%%%%%%%%%%%%%%%%%%%%%%%%%%%%%%%%%

\section{Preliminaries and Problem Definition} \label{sec:pre}
We consider a system of homogeneous disk-shaped DD robots operating in an obstacle-free environment. The left and right wheel speeds of a given robot are individually controlled and denoted by $v_l$ and $v_r$, respectively. A robot's state space is $\dR^2\times S^1$, where a state consists of the $(x,y)$ coordinates of the robot's center, and its orientation $\theta$. The kinematic constraints of a robot are given as the ODE 
%\begin{equation}
    $\left(\begin{smallmatrix}
        \dot{x}\\ \dot{y} \\ \dot{\theta}
    \end{smallmatrix}\right)
    = 
    \left(\begin{smallmatrix}
        \cos \theta & 0 \\ \sin \theta & 0 \\ 0 & 1
    \end{smallmatrix}\right)
    \left(\begin{smallmatrix}
        v \\ \omega
    \end{smallmatrix}\right),$
    %\label{eq:kinematics}
%\end{equation}
where $v = \frac{v_l + v_r}{2}$ is the tangential speed,  $\omega = \frac{v_r - v_l}{d_{iw}}$ is the angular velocity, with inter-wheel diameter~$d_{iw}>0$~\cite{klancar2017wheeled}.\footnote{In our calculation, we use the dimensions of the e-puck robot~\cite{mondada2009puck}. The robots are disc-shaped of radius $r = 3.7[\text{cm}]$,  inter-wheel diameter $d_{iw} = 5.1[\text{cm}]$, and %\royalso{wheel radius of $r_w = 0.41[\text{cm}]$}, with
maximum wheel speed $v_i = 12.8~\left[\frac{\text{cm}}{\text{s}}\right]$ each.}

Each robot is equipped with a binary sensor, which determines whether there is another robot along the infinite ray emanating from the ego robot's center, with orientation $\theta$.
%While the e-puck is equipped with eight IR sensors and a forward-facing camera, we need only make use of the forwards-facing camera in order to implement a line of sight sensor. Said sensor outputs a binary signal depending on the existence of a robot in the direct line of sight. 
If there is no robot in sight, the output is '$0$', and '$1$' otherwise. %These sensors are not used for explicit communication, and in fact no direct communication exists between the robots.

% Problem definition
Before proceeding to the problem statement, we state four assumptions.  (1) Following~\cite{gauci2014self}, the robots are memoryless and cannot communicate. Next, we introduce several assumptions to make the analysis tractable.  (2) The robots cannot push each other upon collision (i.e., all collisions are purely plastic, and no momentum is conserved).  (3) No noise in the motion model or sensor measurements is present.  (4) There is no tire slippage.

Considering  that the robots cannot communicate or compute, a  controller is  of the form $\left(v_{l,0}, v_{r,0}, v_{l,1}, v_{r,1}\right)$, where $v_{l}$ and $v_{r}$ are the controls for the left and right wheels respectively, with the subscripts $1$ and $0$ indicating that another robot is in the line of sight (LoS)  or not, respectively. The controls are normalized by dividing the speed of each wheel by the maximum speed. 
We denote the space of all bimodal controllers by $\U$. Note that as the robots are homogeneous, they all use the same bimodal controller. 

We are interested in designing bimodal controllers that lead to aggregation, which is defined as follows. Given that the state space of an individual robot is $\X:=\dR^2\times S^1$, the state space of the multi-robot system can be described as $\X^n:=\X\times \cdot \times \X$, where $n\geq 2$ is the number of robots. For a given multi-robot (MR) state $x\in \X^n$ we use $x_i\in \X$ to denote the $i$th robot state, where $1\leq i\leq n$. We denote by $\X^n_f\subset \X^n$ the free space, i.e., for any $x\in \X^n_f$ and any two robots $i\neq j$ it holds that $\|x_i-x_j\|\geq 2r$.  

We denote by $\X^n_a\subset\X^n_f$ the set of aggregated MR states. In particular, a MR state $x\in \X^n$ is in $\X^n_a$ if and only if the set $\bigcup_{i=1}^n \D_{r+\rho}(x_i)$ is \emph{connected} in $\dR^2$, where $\D_{r+\rho}(x_i)$ is a $(r+\rho)$-disc centered at $x_i$ for some user-defined value $\rho\ge0$.\footnote{For simplicity, in our proofs we set the padding parameter  zero, but they can be generalized to any value of $\rho\geq 0$.} Our definition of aggregation is weaker than the one considered in~\cite{daymude2021deadlock}, where the discs are required to form a \emph{compact} set. As a result, our impossibility proofs are more general, as they also apply for the compact setting. 
%there exists a connected graph $\G$ in which every vertex denotes a robot in the corresponding physical system. Two vertices $u$ and $v$ are considered connected if in the physical system the distance between the centers of two corresponding robots $u$ and $v$ is smaller than $2r+\rho$, where $\rho$ is \roy{Not sure how to phrase this} used for padding and can be set, for example, to $\rho = 0.01r$. 

Denote by $\pi_{x_0,u}:\dR_{\geq 0}\rightarrow \X^n$ the trajectory describing the motion of the robots over time from an initial state $x_0\in \X_f^n$ (i.e., $\pi_{x_0,u}(0)=x_0$) for a given (bimodal) controller $u\in \U$. We are ready to define our problem, which is designing a bimodal controller that leads to aggregation. 

% \kiril{define formally. One way to do that would be to require that the union of $r+\rho$ discs placed at the center of the robots is connected, but I'm not sure about the exact definition you used.}\roy{The way I see it is that aggregation means that there exists a graph where each node is connected to another node if the distance from its center to another robot's center is exactly $2r$. If this graph includes every node possible (every robot possible), then aggregation has occurred. This does not minimize the dispersion, but does mean that they are all one cluster.} \kiril{Ok, this is basically what I meant. I introduced the $\rho$ value for padding. Saying the discs yield a connected region is the same as saying the resulting graph is connected.} 

%\Crefname{problem}{problem}
\begin{problem}\label{problem:main}
    Find a bimodal controller $u\in \U$ such that for any number of robots $n\geq 2$ and initial MR state $x_0\in \X^n_f$ the system aggregates. That is, for every $x_0\in \X^n_f$ there exists $t\in [0,\infty)$ such that $\pi_{x_0,u}(t)\in \X_a^n$. 
\end{problem}

% \niceparagraph{Discussion.} As mentioned earlier, a previous paper has computed numerically the controller $u_{prev}:=(-0.7,-1,1,1)$ and claimed using the theoretical argument that it is aggregating for $n=2$, and conjectured it is also aggregating for $n\geq 3$ from experimental observations~\cite{gauci2014self}. A recent paper has challenged the above claim by showing a counterexample for $u_{\text{prev}}$  for $n\geq 4$~\cite{daymude2021deadlock}. In our work, we provide a more general negative result by showing that no controller exists that solves Problem~\ref{problem:main} (see \Cref{sec:n>2}). I.e., for any bimodal controller $u\in \U$ there exists a robot number $n\geq 2$ and an initial MR state $x_0\in \X^n$ such that $\pi_{x_0,u}(t)\notin \X^n_a$ for any $t\in [0,\infty)$. For some controllers, we prove even stronger results that nonaggregation occurs for a significant volume of initial MR states. In other words, nonaggregation occurs with strictly positive probability, where the initial state $x_0\in \X^n$ is sampled uniformly at random within a finite space.  Before we proceed to the negative results, we first discuss the case for $n=2$, where aggregation is achievable. 

%%%%%%%%%%%%%%%%%%%%%%%%%%%%%%%%%%%%%%%%%%%%%%%%%%%%%%%%%%%%%%%%%%%%%%%%%%%%%%%%%%%%%%%%%%%%%%%%%%%%%%%%%%%%%%%%%%%%%%%

\section{Aggregation for two robots} \label{sec:n=2}
In this section, we consider the most simple setting of two robots ($n=2$). We first point out an issue in the proof a previous work~\cite{gauci2014self} which claimed to have found an aggregating controller for this case. Then, we provide an alternative controller and formally prove that it aggregates. 

\subsection{Issue in a previous proof and a possible remedy}
A proof was provided for the aggregation of two robots using the aforementioned $u_{prev}=(-0.7,-1,1,-1)$ in~\cite{gauci2014self}. The motion of the DD robots using this controller is divided into two parts. For every robot $i$, when no other robot is in its LoS a backwards, clockwise circular motion is followed. The center of this rotation is referred to as the instantaneous center of rotation (ICR), and the radius of the circle is denoted as $R$. If another robot is in the LoS, the ego robot  rotates clockwise on the spot. %As each robot conforms to either of these types of motion, the complete path for each robot is split into two parts - so long as the other robot '$j$' is not in sight, the ICR does not change. 

The proof proceeds as follows. When robot $j$ enters the LoS of robot $i$, it starts rotating on the spot, and as its ICR is located at a fixed distance and angle in relation to $i$, it moves closer towards robot $j$. While the distance between the ICR of the robots during their motion could not be solved analytically, an expression was found that, when solved numerically, led to a solution depending on the initial distance between the two robots and the angle of robot $j$. 

An implicit assumption is made that if at time $t$ the distance between the two robots satisfies $d \le 2(R+r)$, then aggregation is assured at time $t'>t$. This assumption is sound for a scenario in which the two robots can see one another at some point in time, and begin their aggregation.  However, the two robots can be positioned such that this distance is satisfied, yet the robots do not aggregate (see Fig.~\ref{fig:2 robots circle}).  

\begin{figure}
    \centering
    \frame{\includegraphics[width=0.53\linewidth]{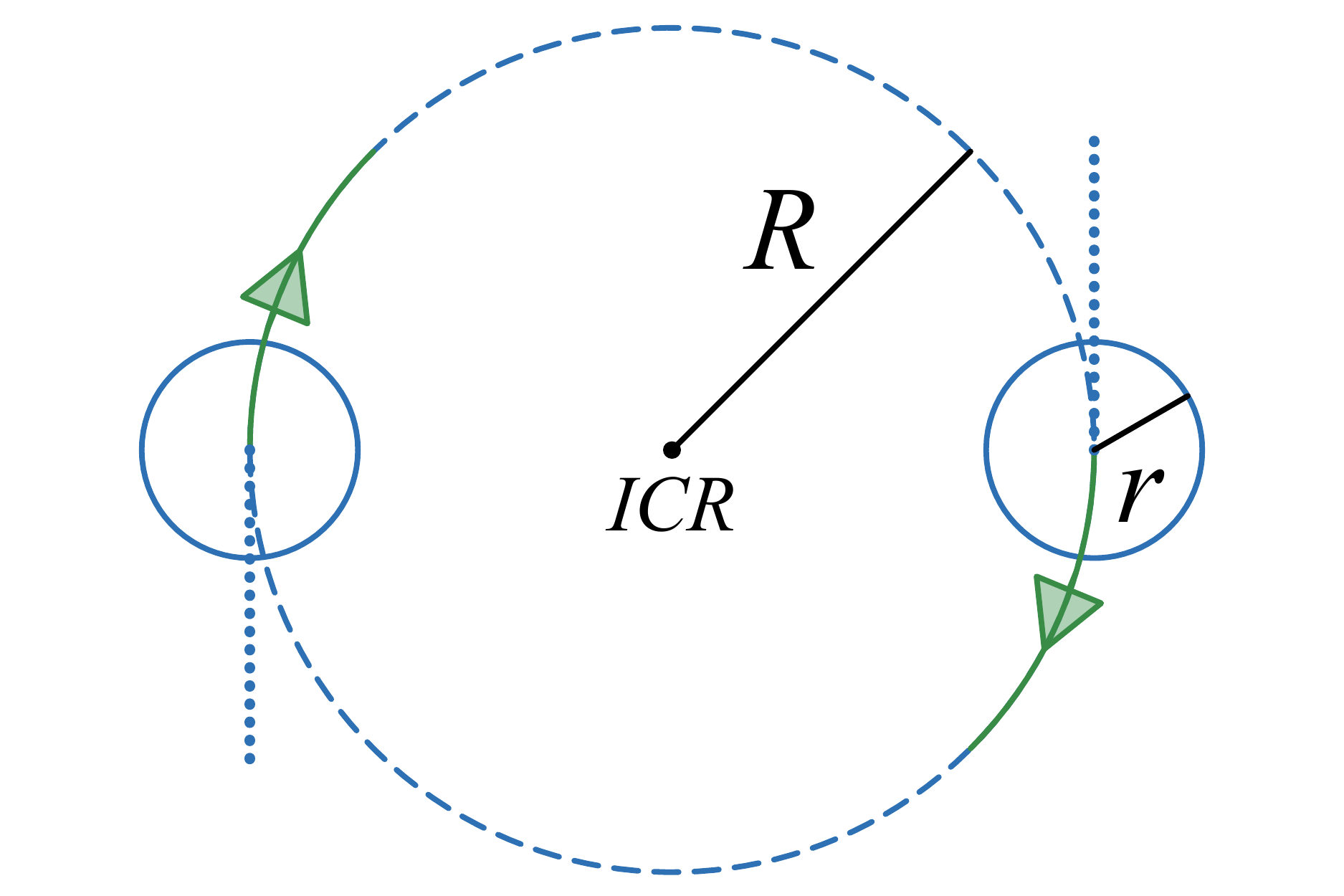}}
    \caption{An example of a no-sensing scenario with two moving robots which violates an assumption made in the aggregation proof for two robots~\cite{gauci2014self}. Both robots move along the dashed circle of radius $R$ centered around the ICR with the same speed and never see one another or aggregate.}
    \label{fig:2 robots circle}
\end{figure}

% An initial assumption is made that the two robots will aggregate so long as the distance between the two satisfies $d \le 2(R+r)$, where $R$ is the radius of the circular motion, and $r$ is the radius of the robots. \kiril{This sentence is unclear. Why is necessary to assume that the robots are close to one another? Isn't the point of the proof to show that it holds for any distance?} This assumption is sound for a scenario in which the two robots can see one another at some point in time, and begin their aggregation. However, the two robots can be positioned such that this distance is not satisfied, yet the robots do not aggregate (see Fig.~\ref{fig:2 robots circle}). 

To circumvent this issue, we suggest to revise $u_{prev}$ into a controller of the form $(-a, -b,1,-1)$ for $b>a>0$ such that the resulting radius of circular motion $\tilde{R}$ is small enough that the two robots cannot simultaneously inhabit the same perimeter of circle with radius $\tilde{R}$. In other words, if $\tilde{R}\leq r$ then both robots will eventually see one another and thus ultimately aggregate. For instance, consider the controller $\tilde{u}_{prev} := (-0.18, -1, 1, -1)$, which yields the radius $\tilde{R} = 3.67~[cm]$. Unfortunately, we do not have a formal proof to show that the revised controller $\tilde{u}_{prev}$ aggregates. We did evaluate this claim experimentally where $\tilde{u}_{prev}$ aggregated for all initial states (see Sec.~\ref{sec:sim}), in contrast to $u_{prev}$ which fails in a significant number of test. 

%and the angular speed $\tilde{\omega}_{0} = -2.058~\left[\frac{\text{Rad}}{s}\right]$. 

%To validate our assumption, we simulate the dynamics of two robots running $\tilde{u}_{prev}$, with their initial positions sampled uniformly within a square of dimensions $\frac{200}{\sqrt{2}}\times\frac{200}{\sqrt{2}}$ and their orientations sampled uniformly in the range $[0,2\pi)$. The simulation platform is further discussed in Sec.~\ref{sec:sim}. $10,000$ simulations were run, and all ended in aggregation. In comparison, simulating $u_{prev}$ with the same possible initial configurations led to $4.24\%$ of all simulations not culminating in aggregation.

%We emphasize that this does not provide formal proof, as we might have overlooked nonaggregating samples for $\tilde{u}_{prev}$. Moreover, some nonaggregating states might be of zero measure, and hence cannot be sampled with positive probability. We defer the formal proof for future work.  

\subsection{An alternative provably-aggregating controller}
We introduce an alternative bimodal controller $u^*$ and formally prove that it aggregates. The controller $u^*$ rotates the robot on the spot in a clockwise manner while no robot  is not in the LoS and moves straight forward otherwise. %For simplicity, we assume that the padding constant $\rho$ is equal to zero.

%\kiril{Consider mentioning that we leave  the time bound for future work.}\roy{Doesn't it look bad that we keep saying that we leave stuff to future work?}
\begin{lemma}
    The bimodal controller $u^*=(-a,a,b,b)$, where $a,b \in (0,1]$ is aggregating for $n=2$. \label{lem:2robots}%Furthermore, the time of aggregation is at most  
    %\begin{equation}
    %    \frac{\left(2\pi-2\sin^{-1}\left(\frac{r}{d_0}\right)\right)}{\omega_0} + \frac{d_0-2r}{2v_0} ~[s],
    %    \label{eq:time bound bimodal new}
    %\end{equation}
    %where $d_0$ is the initial distance between the two robots and $\omega_0$, $v_0$ are the angular and tangential velocities generated from $u^*$, respectively.
\end{lemma}

\begin{proof}
    We apply the controller $u^*$ where both robots are at an arbitrary initial MR state $x_0\in \X^2_f$ at time $t=0$. %Recall that the $\pi:=\pi_{x_0,u^*}$ denotes the resulting MR trajectory, and denote by $\pi^i$ the induced trajectory of robot $i\in \{0,1\}$. 
    We partition the proof into three cases. 
    
    (\emph{C1}): If both robots see one another at time $t$ then they move in a straight line along their LoS. We prove that if this occurs, the robots will eventually aggregate while maintaining a LoS throughout the motion. Without loss of generality, we assume that at time $t$ the heading of robot $0$ is the x-axis $\theta_0(t)=0$, and its position is $(x_0(t),y_0(t))=(0,0)$. As robot $1$ is in the LoS of robot $0$, its $y$ coordinate $y_1(t)$ must be in the range $[-r,r]$. Also denote $(x,y):=(x_1(t),y_1(t))$. Denote by $v_0,v_1\in \dR^2$ the rays corresponding to the headings of robot $0$ and $1$, respectively, at time $t$, which begin at the robots' locations at time $t$ (notice that the orientation or origin of the rays is fixed with respect to time $t$). Denote by $v_i(\tilde{t})$ the location along $v_i$ robot $i$ would reach in time $\tilde{t}\geq t$ had it been moving forward along it from time $t$ (see Fig.~\ref{fig:bimodal proof}).
    
    Denote by $(\ell,0):=v_0(t')$ the first intersection point of $v_0$ with the boundary of robot $1$ for some time $t'\geq t$. Similarly, denote by $(x'',y''):=v_1(t'')$ the first intersection point of $v_1$ with the boundary of robot $0$, and notice that $y''\in [-r,r]$. Without loss of generality, assume that $t'\leq t''$ and denote $(x',y'):=v_1(t')$ (otherwise, we switch the roles between the two robots and transform the system accordingly). Next, we show that if both robots follow $v_0$ and $v_1$ from time $t$ they will maintain visibility until colliding. Observe that robot $1$ will remain visible to robot $0$ since $y,y''\in [-r,r]$. From symmetry, robot $0$ will remain visible for robot~$1$. 
    
    It remains to show that robots $0$ and $1$ will eventually collide with one another as they proceed along $v_0$ and $v_1$, respectively. For the proof below, we allow the robots to overlap and pass each other as they move along $v_0$ and $v_1$, which implies aggregation at an earlier time. As robot $1$ reaches $(x',y')$ before $(x'',y'')$, we have that $y'\in [-r,r]$ as well. Next, notice that if $x'> \ell$ then the two robots must be in collision in time $t$ as $(x',y')\in \D_r(x,y)$, which implies that $\|(\ell,0)\|=\|(x',y')-(x,y)\|<r$.  Thus, $x'\leq  \ell$. Since the $x$ coordinate of robot $0$ monotonically increases between time $t$ and $t'$ along $v_0$ from the values $0$ to $\ell$, and the $x$ coordinate of robot $1$ monotonically decreases along $v_1$ between time $t$ and $t'$ from the values $x$ to $x'$, where $0\leq x'$ and $x'\leq  \ell$, there must be a time $t^*$, where $t\leq t^*\leq t'$, such that $x$ coordinates of both robots are the same. As their $y$ coordinates cannot be more than $r$ apart, the robots are in collision. This implies that at an earlier time than $t^*$, the robots aggregated.

    \begin{figure}
    \centering
    \includegraphics[width=1\linewidth]{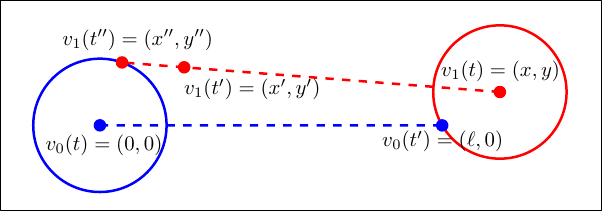}
    % \begin{subfigure}{0.15\textwidth}
    %     \centering
    %     \frame{\includegraphics[width=1\linewidth]{images/bimodal-proof/Simple Bimodal Controller C1.png}}
    %     \caption{}
    %     \label{fig:bimodal proof a}
    % \end{subfigure}
    % \begin{subfigure}{0.15\textwidth}
    %     \centering
    %     \frame{\includegraphics[width=1\linewidth]{images/bimodal-proof/Simple Bimodal Controller C2.png}}
    %     \caption{}
    %     \label{fig:bimodal proof b}
    % \end{subfigure}
    % \begin{subfigure}{0.15\textwidth}
    %     \centering
    %     \frame{\includegraphics[width=1\linewidth]{images/bimodal-proof/Simple Bimodal Controller C3.png}}
    %     \caption{}
    %     \label{fig:bimodal proof b}
    % \end{subfigure}
    \caption{Illustration for Case 1 in the Lemma~\ref{lem:2robots} proof. Robot $0$ and $1$ are in blue and red, respectively, with corresponding rays $v_0,v_1$ as dashed lines. 
    \label{fig:bimodal proof}}
    \vspace{-7pt}
\end{figure}
    
    (\emph{C2}) Now suppose that only one of the robots (without loss of generality, robot $0$) sees the other (robot $1$) at time $t$. Thus, robot $0$ starts moving in a straight line while robot $1$ rotates on the spot. If robot $0$ reaches robot $1$ before the latter sees him, then aggregation occurs. %Notice that in such a case, when robot $1$ eventually sees robot $0$, they will not move anymore. 
    Otherwise, if both robots see each other, then we transfer to C1, which also yields aggregation.
    
    (\emph{C3}): In the final case, none of the robots sees the other at time $t$ and thus rotates in place. Consequently, after a finite $t'>t$, at least one of the robots will see the other, which leads to either C1 or C2. To conclude, for any initial state, the resulting trajectory goes through each of the three cases at most once. 
\end{proof}

% \begin{figure}
%     \centering
% \frame{\includegraphics[width=0.5\linewidth]{images/Two Robots Scenario.pdf}}
%     \caption{An example of a no-sensing scenario with two moving robots which violates an assumption made in the aggregation proof for two robots in~\cite{gauci2014self}. Both robots move along the dashed circle of radius $R$ centered around the ICR with the same speed and never see one another or aggregate.}
%     \label{fig:2 robots circle}
% \end{figure}

%Note that unlike the previous controller suggested, as the two robots travel strictly linearly towards one another, there's no option for one of the robots to 'miss' the other due to it being in its blind spot at a certain point of time, which is the reason $u_{prev}$ could not aggregate. Each robot is either stationary - in which case it will eventually be seen by the other, or moving directly towards the second robot - in which case it will continue to move towards said robot until either reaching it or being seen by it.

We note that a time bound can be obtained for the suggested controller by looking at the worst case scenarios for each of the three cases. By summing the individual time bounds for each case we receive a bound for the aggregation of two robots using $u^*$.

%%%%%%%%%%%%%%%%%%%%%%%%%%%%%%%%%%%%%%%%%%%%%%%%%%%%%%%%%%%%%%%%%%%%%%%%%%%%%%%%%%%%%%%%%%%%%%%%%%%%%%%%%%%%%%%%%%%%%%%

\section{Impossibility of Aggregation for $n>2$ robots}\label{sec:n>2}

While we have shown that there exists at least one controller which aggregates for $n=2$ robots, we now prove that no controller aggregates for all $n\geq 2$.

\begin{theorem}\label{thm:n>2}
    For any bimodal controller $u\in \U$, there exists a robot number $n_u\geq 2$ and an initial MR state $x_0\in \X^{n_u}_f$ such that $u$ does not aggregate.  
\end{theorem}

%\subsection{Proof of Theorem~\ref{thm}}
% \kiril{The proof requires more structure. First introduce the different types of controllers.}

\begin{proof}
    We categorize the various types of movement that a robot makes, and the corresponding controllers. A bimodal controller is defined by one constant controller (A) when a robot is not in sight and a second constant controller (B) otherwise. All possible mode controllers are described in the following table, with the values $a,b \in (0,1]$, where $a\ne b$, representing normalized wheel controls. 

    {\small
    \begin{center}
    \begin{tabular}{| c | c |}
        \hline
             Movement type & Control input \\
            \hline\hline
            Straight Forwards (SF) & ($a$, $a$) \\
            \hline
            Straight Backwards (SB) & ($-a$, $-a$) \\
            \hline
            Circular Forwards (CF) & ($a$, $ b$) \\
             \hline
            Circular Backwards (CB) & ($-a$, $ -b$) \\
             \hline
            Rotate on Spot (RS) & ($\pm a$, $\mp a$)\\
             \hline
            Stand Still (SS) & ($0$, $0$) \\
             \hline
        \end{tabular}
    \end{center}}
        
We have 36 different possible categories of bimodal controllers. We identify a specific category by concatenating the labels of the A controller and the B controller. E.g., the label SF-CB represents bimodal controllers moving straight forward without a robot in sight, and circularly backward otherwise. XX denotes any of the controllers for A or B. 
    
    Next, we consider all bimodal controller categories, for which we provide counterexamples (numbers in the parentheses indicate the additional controllers types eliminated).
    
    \niceparagraph{CB-RS (1).} For this type of controller, deadlocking scenarios have been shown to guarantee nonaggregation~\cite{daymude2021deadlock} for any $n_u\geq 4$, as we discussed in Sec.~\ref{sec:intro}. Although no-slippage is one of the assumptions we make, which allows to guarantee that deadlocking persists, we provide a slightly more complex counterexample robust to slippage (see Sec.~\ref{sec:sim}). A ring of robots, facing outwards radially, are placed in a ring centered on the origin, while one robot is placed on the origin (see Fig. \ref{fig:bot in the middle}). The outside robots attempt to move backwards, as another robot is not currently in sight, but are blocked by their neighbors and thus cannot move. The robot situated on the origin always has another robot in its LoS, and thus continues to rotate on the spot. As such, we can guarantee that aggregation will not occur.

    \niceparagraph{SS-SS, SS-XX, XX-SS (11).} Any controller, including a stationary controller in either control scheme, trivially, cannot promise aggregation, even for $n_u=2$ robots. We place the robots such that they are all looking at one another (for the XX-SS scenario) or that none of the robots see another (for SS-XX), ensuring the robots will be frozen in place.
    
    \niceparagraph{XX-SF, XX-CF (10).}
    We reuse the deadlocking scenario, but this time the robots in each pair face each other. This scenario promises nonaggregation independent of control scheme A, as the robots cannot move from their initial, nonaggregated, positions. %This approach required an even number of robots, but like in the deadlocking scenario, it can be extended by adding another robot so that it is deadlocked with an existing pair.
    
    % \kiril{Mention $n$ necessary.}
    
    \niceparagraph{XX-SB  (5).}
    If control scheme B is purely backwards, nonaggregation can be promised by having two robots face each other initially. As the LoS is infinite, both robots will drive backwards eternally, regardless of control scheme A. We can continue adding more pairs of robots situated horizontally to existing pairs. As the robots move on the same line, their LoS always includes another robot, promising that they continue their movement and never aggregate.
    
    % \kiril{Explain how to generalize it an uneven number of robots. E.g., Place another robot behind of the two on the same vertical line.}
    
    % \kiril{Continue with similar updates below. The following text is imprecise.} 
    
    % \begin{table}
    %     \caption{Robot movement types and their respective control sequences. The values $a,b \in (0,1]$, where $a\ne b$, are the normalized wheel controls.}
    %     \centering
    %     \def\arraystretch{1.5}
    %     \begin{tabular}{| c | c |}
    %     \hline
    %         Movement type & Control input \\
    %         \hline\hline
    %         Straight Forwards (SF) & ($a$, $a$) \\
    %         \hline
    %         Straight Backwards (SB) & ($-a$, $-a$) \\
    %         \hline
    %         Circular Forwards (CF) & ($a$, $ b$) \\
    %          \hline
    %         Circular Backwards (CB) & ($-a$, $ -b$) \\
    %          \hline
    %         Rotate on Spot (RS) & ($\pm a$, $\mp a$)\\
    %          \hline
    %         Stand Still (SS) & ($0$, $0$) \\
    %          \hline
    %     \end{tabular}
    %     \label{tab:n robots control}
    % \end{table}
    
    \niceparagraph{XX-CB  (5).}
    We expand on the CB-RS counterexample, and develop a counterexample where $n-1$ robots are oriented in a ring, with another robot in the middle. The robots in the ring are positioned such that each robot sees its neighbor to the right, and due to their arrangement they maintain their states, and consequently the ring structure, throughout the controller execution (see Fig. \ref{fig:bot in the middle forwards}). Thus, the robot in the middle always maintains visibility with one of the robots encircling it. Now, if the size of the ring is small enough with respect to the radius $R$ of the circle of motion induced by the CB controller, the lone robot could eventually reach the ring robots, which would result in an aggregation. Thus, we specify the ring of robots to be large enough such that no aggregation occurs. In particular, for a given radius $R$ we can determine geometrically a minimal number of robots to form a bounding ring.

    \begin{figure}
        \centering
        \subfloat[\label{fig:deadlocking}]{
                \frame{\includegraphics[width=0.32\linewidth]{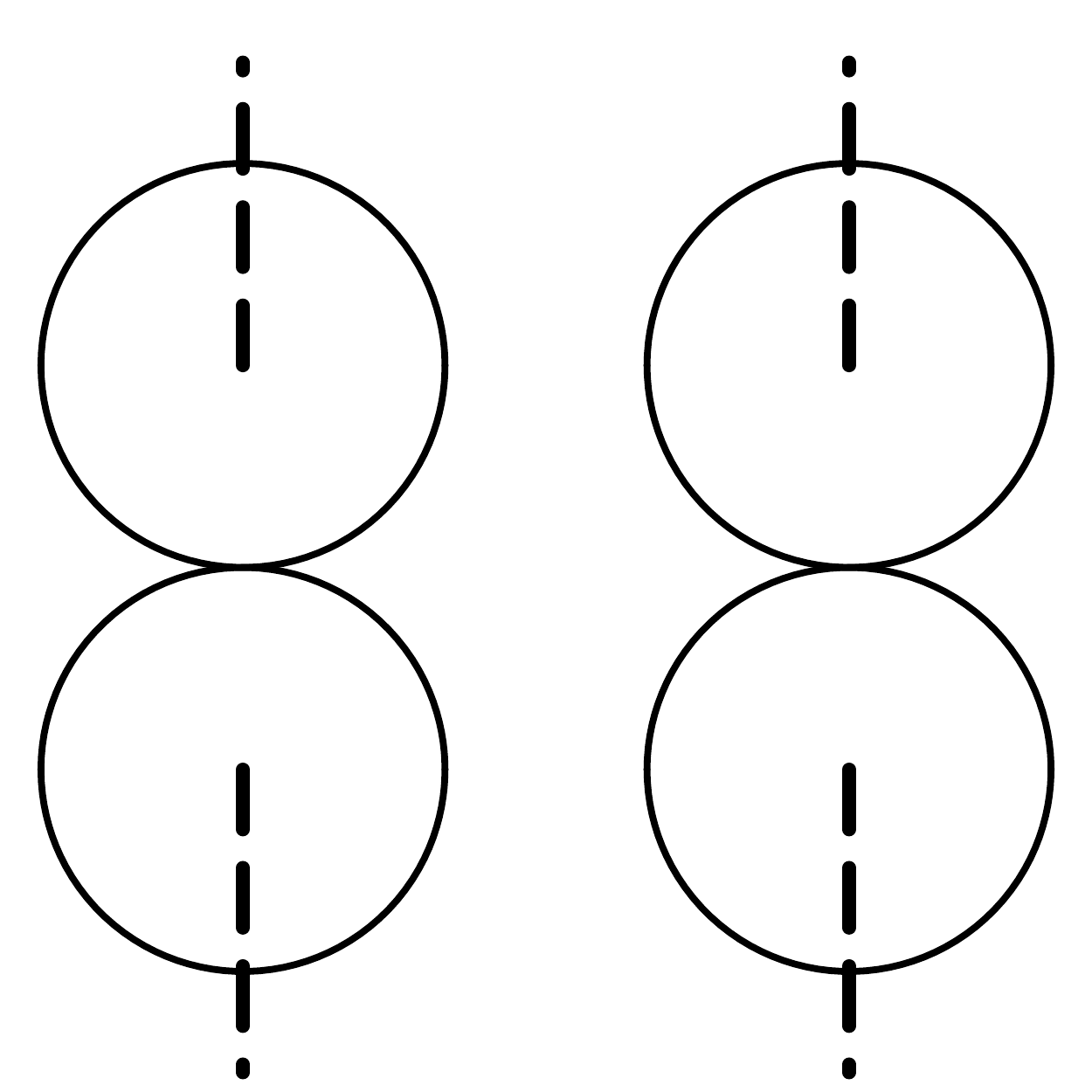}}}
            %\hfill
        \subfloat[\label{fig:bot in the middle}]{
                \frame{\includegraphics[width=0.32\linewidth]{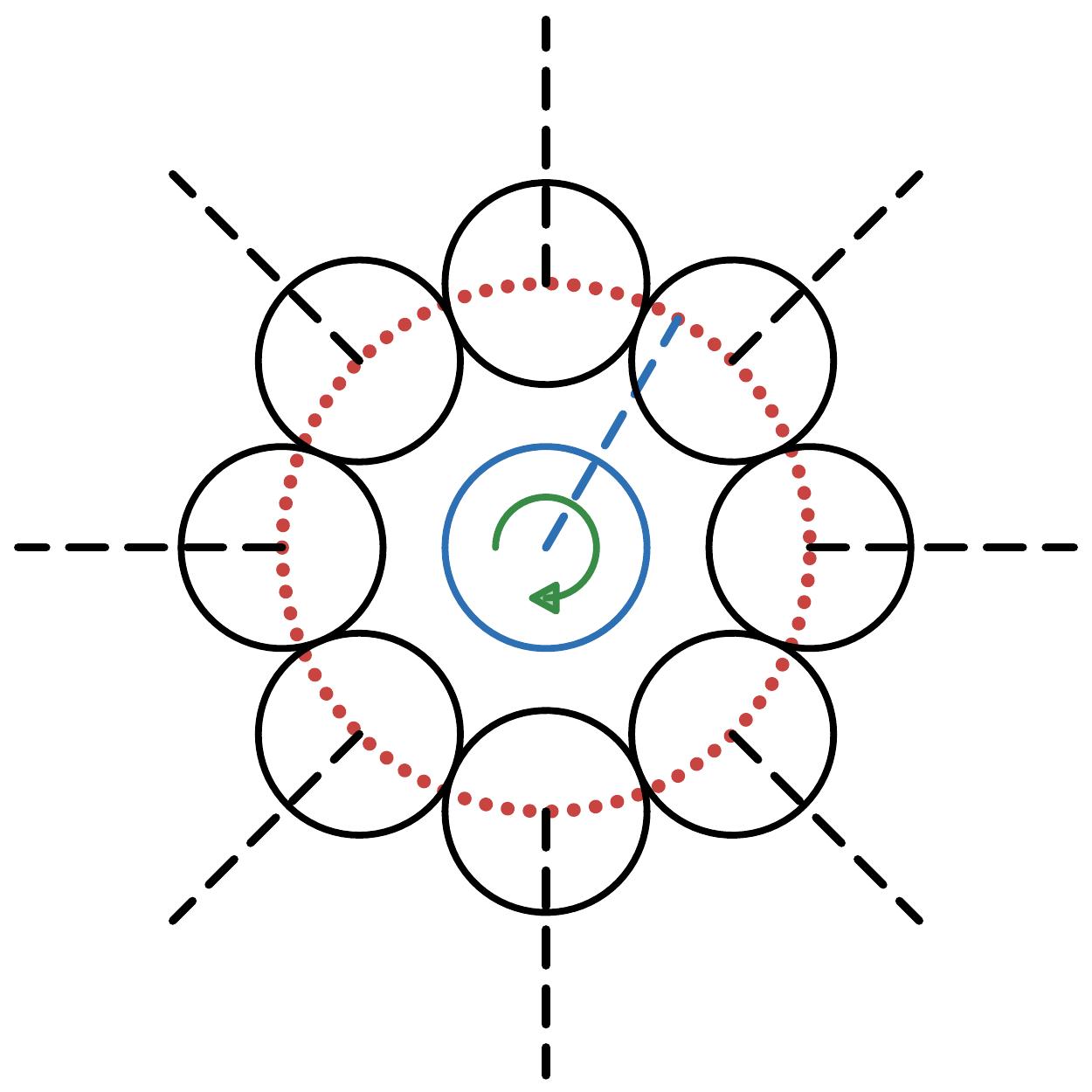}}}
            %\hfill
        \subfloat[\label{fig:bot in the middle forwards}]{
                \frame{\includegraphics[width=0.32\linewidth]{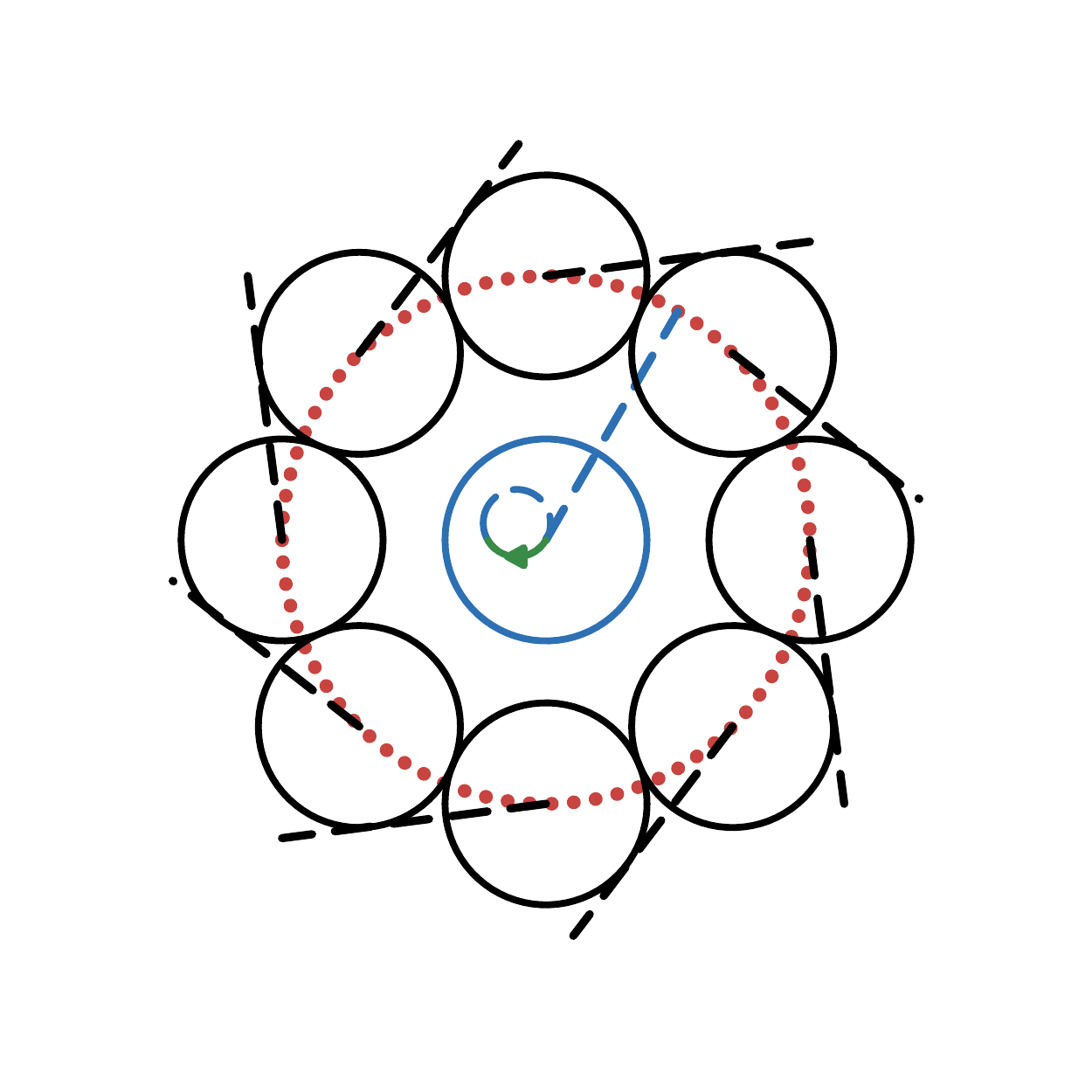}}}
            %\hfill
        \caption{Counterexamples for bimodal controllers. (a) Based on~\cite{daymude2021deadlock}, each robot equipped with an $\{\text{SB,CB}\}$-XX has no other robot in sight and is blocked by another robot from behind. (b) A different counterexample for a CB-RS controller, where  the center robot is stuck continuously rotating. (c) A XX-CB controller causes all peripheral robots to be stuck, as they attempt to move backwards but are blocked by their neighbor. %The number of robots in the ring depends on the center robot's motion.
        }
        \label{fig:Theorem 1 graphs}
         %\vspace{-7pt}
    \end{figure}

    \niceparagraph{XX-RS  (4).} An RS-RS controller is trivially nonaggregating. For $\{\text{SB,CB}\}$-RS, we have shown for CB-RS controllers (Fig.~\ref{fig:bot in the middle}) that when the ring robots are initially orientated outwards radially, these robots are blocked by their neighbors. Finally, for a $\{\text{SF,CF}\}$-RS controller we can place each robot in an orientation very similar to the XX-CB controller case, but this time the robots just barely miss each other and do not shift to control B. As they attempt to move forwards they are blocked by the neighbor in front, in the same manner that an XX-CB controller is blocked from behind. 
\end{proof}

%%%%%%%%%%%%%%%%%%%%%%%%%%%%%%%%%%%%%%%%%%%%%%%%%%%%%%%%%%%%%%%%%%%%%%%%%%%%%%%%%%%%%%%%%%%%%%%%%%%%%%%%%%%%%%%%%%%%%%%

\section{Nonsingularity of nonaggregating states}\label{sec:N deadlock}
The counterexamples for nonaggregation of $n>2$ robots so far require a particular configuration of the initial robot states. Applying random perturbations, which slightly alter the robots' positions and orientations, might result in aggregation. In this section we consider the question of whether all nonaggregating initial states are singular states, that is, the probability of choosing those states is equal to zero, when the robots' states are sampled uniformly at random from a finite domain. We first answer this question in the negative for a specific subset of the controller group CB-XX which includes the controller $u_{prev}$.

% \begin{theorem}
%     \label{thrm:N deadlock}
%     The probability of receiving a nonaggregating initial configuration from a random distribution and a general number of robots is bounded by below by \eqref{eq:N deadlocking}
%     \begin{equation}
%         p = \frac{ R_0^2 \left(3\rho^2_0g_\theta\right)^{\left\lfloor\frac{N}{2}\right\rfloor-1}  \left(\varepsilon^{max}_\theta\norm{\Vec{\varepsilon}^{max}}^2\right)^{\left\lceil\frac{N}{2}\right\rceil}}{8^{\left\lfloor\frac{N}{2}\right\rfloor-1}\pi^{\left\lceil\frac{N}{2}\right\rceil-1}S^N_\text{Circ}},
%         \label{eq:N deadlocking}
%     \end{equation}
%     where $N$ is the number of robots, $R_0$ is a user-defined parameter, and $S_{\text{Circ}}$ is the area of the largest bounded circle inside the two-dimensional space. The rest of the parameters are user-defined under certain restrictions specified in this Theorem.
% \end{theorem}

\begin{theorem}
    \label{thrm:N deadlock}
    Consider a \textup{CB-XX} controller $u\in \U$ such that control A induces a circular motion of radius $R>1.707r$. Fix an even number of robots $n\geq 4$. Then the probability of nonaggregation of $\pi_{u,x_0}$ over the uniform random 
    choice of the initial MR state $x_0$ from a large enough workspace is strictly positive. 
\end{theorem}

% \kiril{The proof looks sound but requires more precision. E.g., are you referring to a the specific controller of Gauci or a family of controllers? If it's the former, can it be generalized? Can you extend it to an uneven number of robots? In your description of the construction, be more explicit. E.g., we choose positions for robot $1,2$. ... Next we choose positions for $3,4$, etc.}

\begin{proof}
    % Sampleable area
    %Before placing our robots we wish to find a subset of the area so as to create a more comfortable configuration space. There is no assumption made on the type of two-dimensional space given, so we limit the possible sampled area to a bounded circle with the largest radius ($\B$) such that $\X'^N_f=\max\{\}$\roy{I don't know how to write this :(}. This circle is of area $S_\B = \pi R_\B^2$.
%Consider once again the scenario of deadlocking for a general, even number of robots, with \royalso{a CB-XX controller. For reasons explained later, we limit the controller to a subgroup in which the radius of the circular motion satisfies $R>1.707r$, of which $u_{prev}:=(-0.7,-1,1,1)$ is a part of}. To deadlock two robots they do not necessarily need to be vertical to one another. So long as their path is blocked by the other, they are aggregated and cannot move. \royalso{Designating one robot as robot $0$}, we can look backwards at the trajectory of \royalso{its adversary ('$1$')}, so that it reaches an end point where it is deadlocked. In this case, we assume that the robot $0$ can be placed anywhere, as robot $1$ needs only be placed in relation to it. Furthermore, we can find areas in which more pairs of deadlocked robots can be placed, so that no interaction between any of the pairs occurs (i.e. no robot falls into the line of sight of another), and as such cannot aggregate.  
    We are looking for a family of MR states $x_0\in\X^n_f$, such that the robots are grouped into pairs, and the following conditions are satisfied: (i) No pair of robots will see another pair at any time.
        (ii) No pair of robots will be seen by another pair at any time.
        (iii) All pairs will end up individually deadlocked.
    
    % Preliminary for deadlocking
    To answer the first item, we wish to find an initial configuration that ensures no robot will see more than a quarter of its visible plane during its motion. As every pair of robots must arrive in a deadlock, we start by looking at the possibilities of deadlocking the first two robots $0$ and $1$. Consider an initial scenario in which the ICR of robots $0$ and $1$ are situated at $(-R,0)$, $(R,0)$, respectively. If placed at opposite initial orientations %\remove{(difference of $\pi$ between the two)}, 
    and ignoring collisions, the two robots reach the origin at time $t$. At this time, the distance between each robot and the origin is zero, and the overlap between the two is their entire area---a circle with a radius of $r$. As we are excluding collisions, the two robots will stop before reaching this point, but this does demonstrate that if at time $t$ there exists an overlap between the two robots (if ignoring collisions), then a deadlock must have occurred at time $t'<t$.%\remove{ From the symmetry of this configuration, the point of deadlocking and the center of the overlap are identical.}

    % Sampling of robots 0 and 1 - introduction of perturbations
    To ensure that no more than a eighth of the plane is visible during the motion between time $0$ to $t$, we set the initial position of robot $0$ to $(x_0,y_0) = \left(-R\left(1-\frac{1}{\sqrt{2}}\right),\frac{R}{\sqrt{2}}\right)$, or an eighth of an arc backwards from the point of deadlocking (see Fig. \ref{fig:N deadlock}). Specifically, robot $0$ can see points $(x,y)$ that satisfy $\{\max{(x,-x-y)}\le0\}$. As robot $1$ is set to the same arc, its position is simply $(x_1,y_1)=-(x_0,y_0)$. Similarly, this robot sees the set $\{\max{(-x,y+x)}\le0\}$ from time $0$ to $t.$ The radii $R$ which enable this are all $R>\frac{r}{2-\sqrt{2}}>1.707r$, as the distance between the initial robots $0$ and $1$ states must be at least $2r$. Using this, the starting orientations of robots $0$ and $1$ are $\theta_0:=\frac{3\pi}{4}$ and $\theta_1:=-\frac{\pi}{4}$, respectively and with respect to the positive x-axis.

    Next, we define a maximal perturbation $\eps^*=(\eps^*_x,\eps^*_y,\eps^*_\theta)$ applied to the initial states of robots 0 and 1, which will be subsequently also applied to the initial states of the following robots. In particular, robot $i$'s initial configuration is updated to be 
    $\bm{x}_i:=(x_i,y_i,\theta_i)+(\eps_{i,x},\eps_{i,y},\eps_{i,\theta})$, where the perturbations are chosen $\eps_{i,x}\in (-\eps^*_x,\eps^*_x),\eps_{i,y}\in (-\eps^*_y,\eps^*_y),\eps_{i,\theta}\in (-\eps^*_\theta,\eps^*_\theta)$ uniformly at random, and independently between the robots. We choose the values $\eps^*$ such that (1) each robot arrives at time $t$ to a position that is at distance at most $r/2$ from the origin (while ignoring collisions), (2), along their motion robot $0$ and $1$ see at most the points $(x,y)\in \dR^2$ that satisfy $\{\max{\left(\tan{\left(\frac{7\pi}{8}\right)}(x+2r), \tan\left(\frac{3\pi}{8}\right)(x-r)\right)}-y\le0\}$ and $\{y+\max\left(-\tan\left(-\frac{\pi}{8}\right)\left(x-2r\right),-\tan\left(\frac{3\pi}{8}\right)\left(x+r\right)\right)\le0\}$, respectively (see Fig. \ref{fig:N deadlock}). In particular, by fixing 
    \begin{align*}
        \eps^*=\left(\frac{r}{8},\frac{r}{8}, \min{\left\{\frac{\pi}{8}, \arccos{\left(1 -  \frac{64r^2 - r^2}{64R^2\left(2 - \sqrt{2}\right)}\right)}  \right\}}\right),
    \end{align*} 
    those conditions are satisfied. This ensures that the robots are still in collision at time $t$, which implies that they stopped after meeting each other at time $0<t''<t$. Moreover, this would allow us to position additional robots in the areas not visible by the first two robots, and reuse the same arguments. 

    % \royalso{$\{\max{\left(\tan{\left(\frac{3\pi}{4} + \varepsilon^*_\theta\right)}x-y, \tan{\left(\frac{\pi}{2}-\varepsilon^*_\theta\right)}x-y\right)}\le0\}$ and $\{\max{\left(y-\tan{\left(-\frac{\pi}{4} - \varepsilon^*_\theta\right)}x, \tan{\left(-\frac{\pi}{2} + \varepsilon^*_\theta\right)}x-y\right)}\le0\}$}

    \begin{figure}
        \centering
        \subfloat{
                \frame{\includegraphics[width=0.47\linewidth]{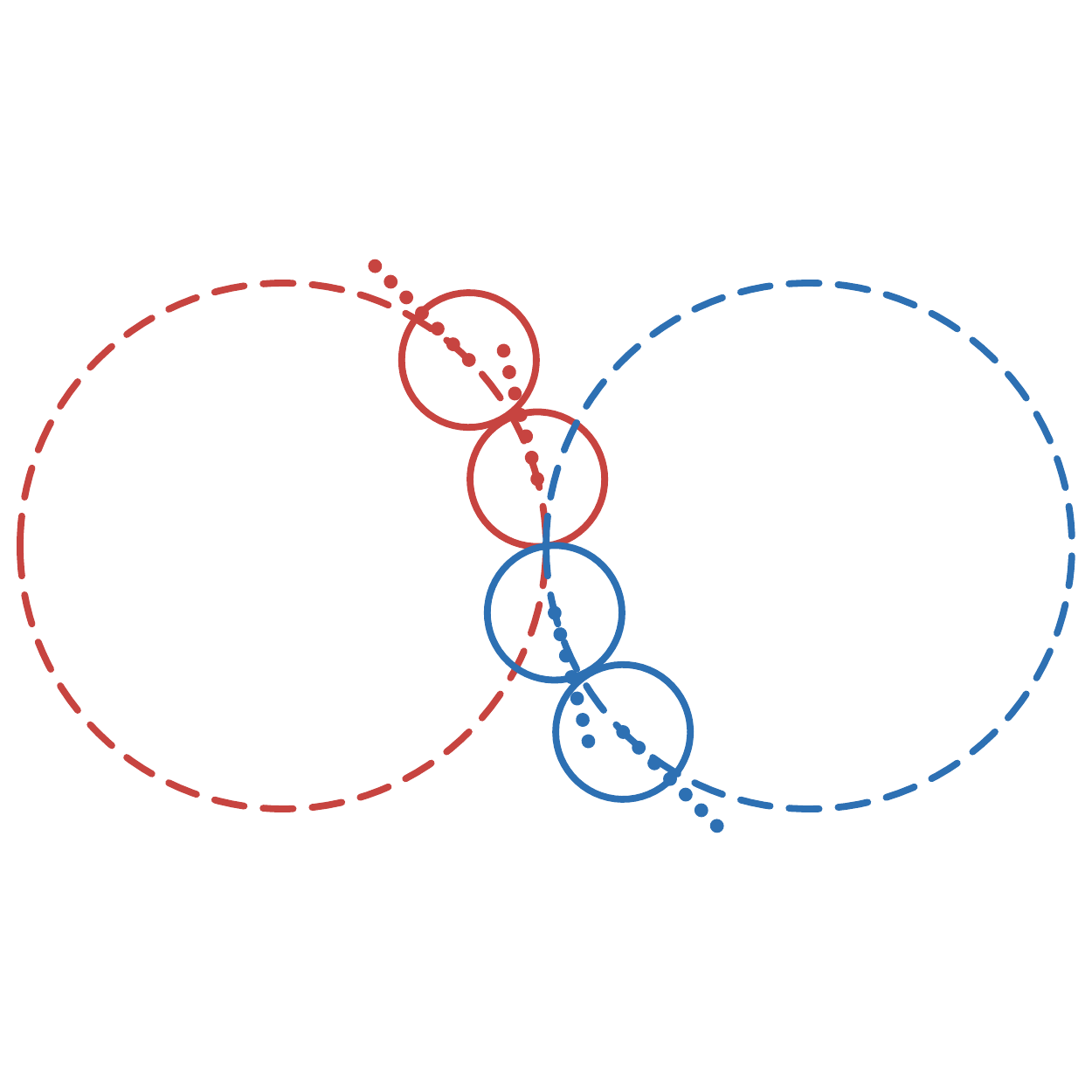}}}
            %\hfill
        \subfloat{
                \frame{\includegraphics[width=0.47\linewidth]{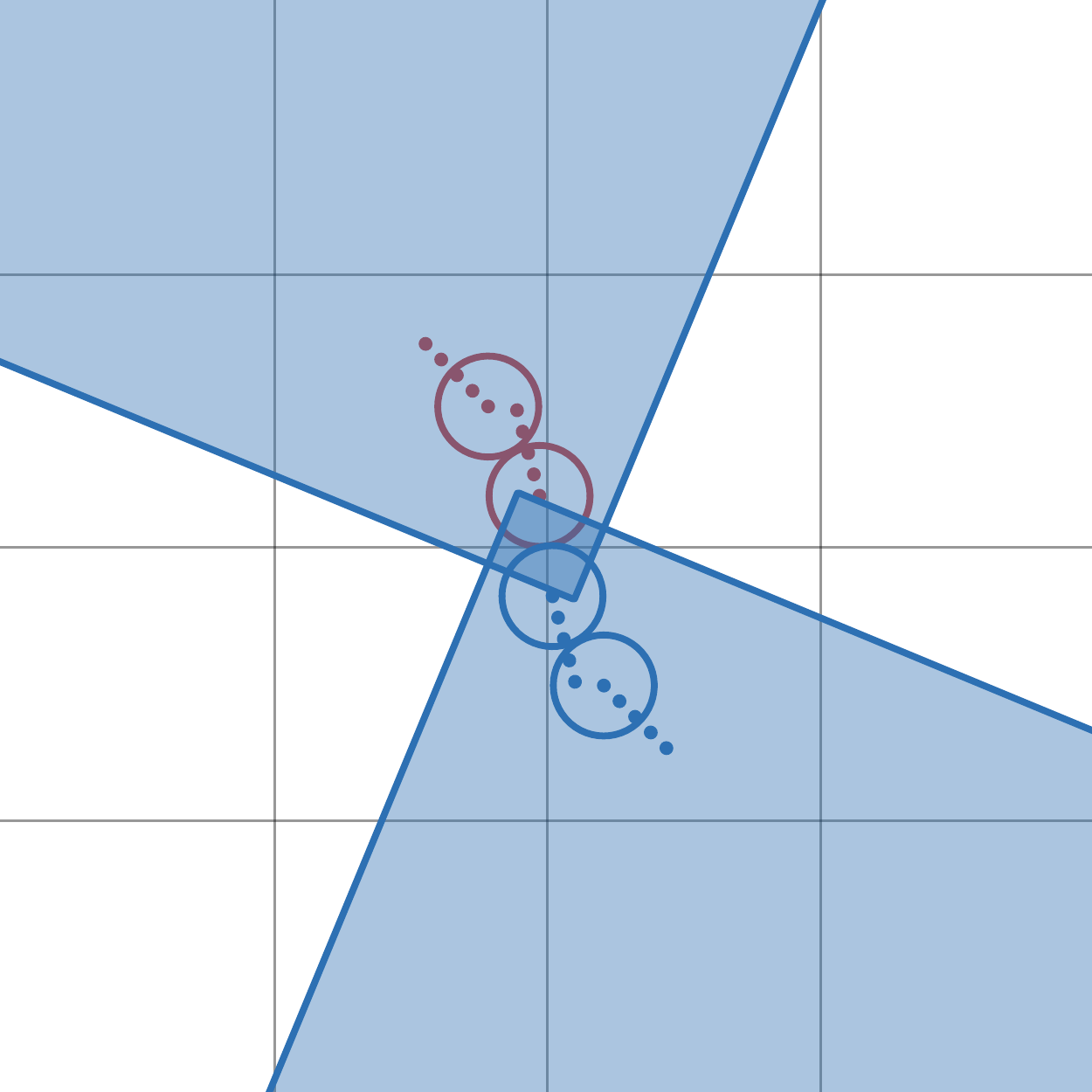}}}
            %\hfill
        \caption{Motion of the first pair of robots in Theorem~\ref{thrm:N deadlock}. [left] The robots meet at a deadlock at the origin if placed without perturbations. [right] With added perturbations, the two robots see no more than the highlighted area.}
        \label{fig:N deadlock}
         %\vspace{-7pt}
    \end{figure}
    
    While we omit the technical details for obtaining $\eps^*$ we provide some intuition. We choose $\eps^*_x$ and $\eps^*_y$ such that the location of each robot is at a distance of at most $r/4$ from the origin at time $t$, assuming that $\eps^*_\theta =0$. Next, we increase $\eps^*_\theta$ such that the distance from the origin increases from $r/4$ to at most $r/2$.  Here we note that the additional translation to the position in time $t$ by setting $\eps^*_\theta$ to be nonzero is expressed by a series of rotation and translation matrices, given by $T\left(\bm{x}_i\right)R(\varepsilon^*_\theta)T\left(-\bm{x}_i\right)\bm{d}$, where $T(\bm{v})$ are spatial translations dictated by the vector $\bm{v}$, $R(\phi)$ are anti-clockwise rotations of the origin by an angle of $\phi$, $\bm{x}_i$ is the unperturbed initial position of robot $i$, and $\bm{d}=(0,0,1)^T$ is the origin.

    The change in the orientation of the robots due to the introduction of $\eps^*_\theta$ is obtained by padding range of orientations without perturbation with $\eps^*_\theta$. Thus, each robot covers a portion of orientations which is at most $\frac{\pi}{4} + 2\eps^*_\theta$. As we want each robot to see no more than a quarter of the plane, this value must be bound by $\frac{\pi}{2}$, so that $2\eps^*_\theta < \frac{\pi}{4}$.
    %Altogether, $\varepsilon_\theta$ must be smaller than both this bound, and the bound introduced by \eqref{eq:N deadlock final point after perturbation}, which can be seen combined in \eqref{eq:N deadlock epsilon theta}

    % Sampling robots 2 and 3
    We now proceed to position robots $2$ and $3$ in a similar manner. For simplicity, we rotate the entire plane anti-clockwise around the origin by an angle of $\eps^*_\theta$, to ensure that the visibility of robots $0$ and $1$ is limited to the upper-left and bottom-right quarters respectively (see Fig. \ref{fig:N deadlock}). We  place the next two robots in the same manner as the first pair, by choosing a 'shifted origin' $(s_o,s_o)$    
    for their overlap, situating the ICR of the robots at $(-R+s_o,s_o)$ and $(R+s_o,s_o)$, respectively, and choosing their initial positions at an eighth of an arc backwards on their circles with an addition of perturbation sampled according to $\eps^*$, and by an additional angle of $\eps^*_\theta$ so as to coincide with the rotated plane. By choosing the value $s_o$ to be large enough (but bounded), we can ensure that none of the robots $0$ and $1$ crosses to the visibility region of robots $2$ and $3$, and vice versa. Moreover, for large enough $s_o$ no collision between a robot $i\in \{0,1\}$ with a robot $j\in \{2,3\}$ would occur. 

    This construction can be repeated for any number of robot pairs, as we ensure that the next pair can be positioned in an unseen quarter of a plane, assuming that the workspace $S$ is large enough.    
    Thus, the probability of sampling a nonaggregating MR state is at least $\left(\frac{2\pi r/4}{|S|}\cdot \frac{\eps_\theta^*}{2\pi}\right)^n$, where $|S|$ is the volume of $S$. 
\end{proof}

\niceparagraph{Discussion. }
We believe that similar proofs can be obtained for other types of controllers. For instance, for the  XX-SS controller (and similarly for SS-XX), one can perturb the counterexample presented in Sec.~\ref{sec:n>2} and still ensure that each robot sees another, and so the robots stand still without aggregating. Similarly, perturbing the counterexamples for XX-SF or XX-CF leads to nonaggregation.

In contrast, for XX-SB our counterexample from Sec.~\ref{sec:n>2} breaks when perturbations are included, as pairs of robots would lose their lines of sight, unless their headings are aligned. Counterexamples that involve a ring structure, such as for XX-CB and  XX-RS require more careful examination, which we leave for future work. We do consider empirically the effect of perturbations on the ring structure for the CB-RS controller in Sec.~\ref{sec:sim}.

\begin{figure}
    \centering
\frame{\includegraphics[width=1\linewidth]{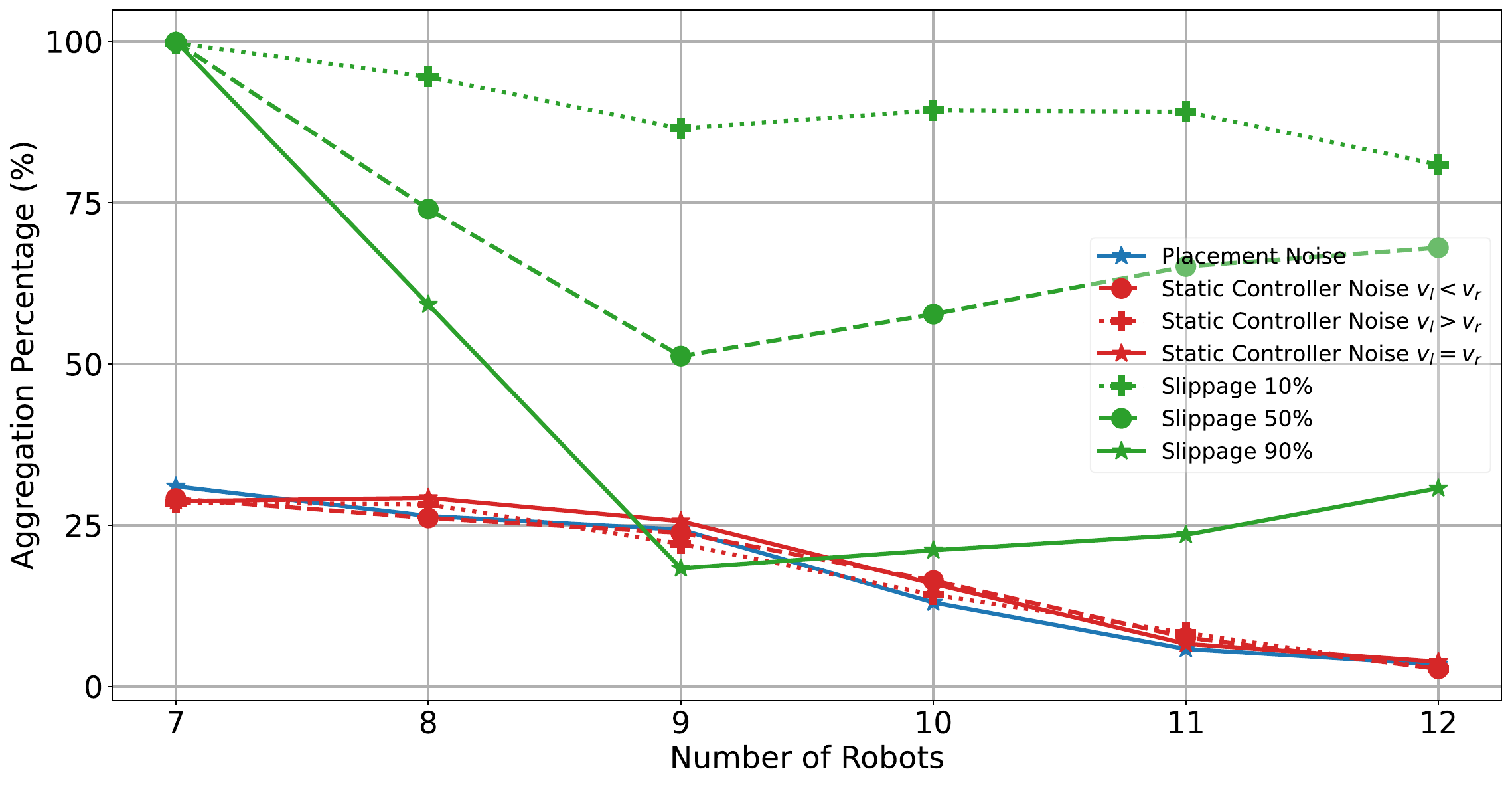}}
    \caption{A plot demonstrating the effect of various noise sources and slippage on the aggregation percentage in the bot-in-the-middle scenario.}
    \label{fig:sim agg percent}
     %\vspace{-7pt}
\end{figure}

\section{Experimental Results}\label{sec:sim}
We present experimental results. We first compare our new controller $u^*$ for $n=2$ with $u_{prev}$ and $\tilde{u}_{prev}$. Next, we consider our 'bot-in-the-middle' counterexample for the XX-CB controller and test its resilience to perturbations, slippage, and control noise. The evaluation was performed using our Python environment where DD robot kinematics and slippage dynamics are simulated. A padding value $\rho$ of $1/20$ of the robot's radius was used throughout.  

\subsection{Aggregation for $n=2$}\label{sec:sim n=2}
% Simulation for bimodal controller $u_prev$
In Sec.~\ref{sec:n=2} we suggested the controller $\tilde{u}_{prev}$ to remedy issues with $u_{prev}$ from~\cite{gauci2014self}. Here we compare its performance against our provably-aggregating $u^*$, as well as with the faulty $u_{prev}$.
We simulate the dynamics of two robots, with their initial states sampled uniformly within a square of dimensions $\frac{200}{\sqrt{2}}\times\frac{200}{\sqrt{2}} \left[\text{cm}^2\right]$, where there are  $10^4$ initial MR states in total. As expected due to Lemma~\ref{lem:2robots}, the robots aggregated for all initial states using the controller $u^*$. Interestingly, $\tilde{u}_{prev}$ achieved a $100\%$ success rate as well, which suggests that it could be provably aggregating. In contrast, $u_{prev}$ failed to aggregate on $4.24\%$ of the tests (either due to the robots performing periodic motions after a certain period of time, or exceeding a time budget of $5\times 10^3$ seconds). %\kiril{How long each simulation lasted?}\roy{Like in real time? I didn't benchmark this, but I know how long it took for them to aggregate.} \kiril{It's an important detail. How did you decide that you've ran the program long enough to declare failure?}\roy{Oh I see. I have a few methods of early stopping, and also a timeout of 500,000 ticks which no simulations reached. The early stopping method checks what the minimal omega > 0 is for all robots, then calculates the period of this angular velocity, and stops the program early if no change in this angular velocity was measured for 1.1 periods at least.}
% Simulation for bimodal controller $u^*$ and time bound.
%For our proven controller $u^*$, we simulated using the same initial states as above. Once again, all simulations ended in aggregation. %A graph of the aggregation times can be seen in Fig. \ref{fig:sim ustar time}.

\subsection{Nonaggregation for $n>2$ robots}\label{sec:sim n>2}
Our theoretical analysis in Theorem~\ref{thm:n>2}  assumes that no slippage or control noise are present. However, we conjecture that this result can be generalized even when those factors are introduced, at least for some forms of bimodal controllers. Here, we test this claim empirically specifically for a XX-CB controller and the bot-in-the-middle counterexample. 

Before we proceed to tackle those questions, we first test whether this setting is also resilient to perturbations. We show that for small perturbations a substantial number of the initial states will not aggregate for the $\tilde{u}_{prev}$ controller. First, assuming no controller noise, we simulate a ring with $6$ to $11$ robots, with an additional robot in the middle of the ring. From this initial configuration we first move each peripheral robot $\frac{r}{2}$ radially outwards, and then we perturb each robot spatially by $\pm\frac{r}{4}$ in both $x$ and $y$ direction, and its orientation by $\pm\frac{\pi}{32}$. The center robot is perturbed similarly with respect to the origin. %This way we generate $10^3$ scenarios in total.

Next, we describe our tests for control noise and slippage. We consider a 'static' controller noise, in which the speed of every wheel individually is slightly changed according to a random number sampled in a 'high' ($0.02$) or 'low' ($0.01$) setting. Three different scenarios were evaluated: (i) The left wheel was given a high amount of noise, while the right wheel was given a low amount, (ii) the opposite of the previous scenario with respect to each wheel, and (iii) both wheels were given a high amount of noise. Additionally, we integrate slippage by simulating conservation of momentum, and adding a variable controlling the percentage of the momentum being conserved, with $0\%$ being purely plastic collisions, and $100\%$ being purely elastic collisions. For slippage also, three scenarios were evaluated, with percentages of $\{10, 50, 90\}\%$ being tested.

% Next, we describe our tests for control noise and slippage. We consider both a 'static' controller noise, in which the speed of every wheel individually and of each sub-controller for each robot is slightly changed by a random amount, and a 'dynamic' noise, which changes the controller by a random amount at each time step. \kiril{What are the amounts of noise?}
% Finally, we integrate slippage with a variable controlling the percentage of the slippage, with $0\%$ being purely plastic collisions, and $100\%$ being purely elastic collisions. Three scenarios were evaluated, with slippage percentages of $\{10, 50, 90\}\%$. \kiril{Can you briefly explain how slippage was simulated?} %It should be mentioned that this simulation does not make use of a physical engine, and thus the results may not be fully accurate. 

%Assuming we are free to place the center robot anywhere and shifting the origin to its position, and with any angle, we receive an initial configuration with a positive probability of being sampled from a uniform distribution of a given area.

%Adding to this configuration noise, we relax some of our earlier assumptions introduced in Sec.~\ref{sec:pre} and simulate both controller noise and slippage individually. For controller noise two scenarios are considered. 

Fig. \ref{fig:sim agg percent} shows the probability of aggregation for each of the settings mentioned above, where for each robot number $10^3$ scenarios (or runs) were generated. Most importantly, for all scenarios aggregation was not promised, with the percentage of aggregating scenarios generally falling  with the increase in the number of robots. Interestingly, as we approach purely elastic collisions, we see that less scenarios aggregate. We believe this is due to the rings forming more evenly when slippage is allowed, and thus the center robot is fully enclosed inside the ring faster and more often in these scenarios.

\section{Future Work}
Our work mathematically proves that constructing an aggregating controller for a large number of robots is an impossible task. Moreover, our empirical work, as well as that in previous work~\cite{daymude2021deadlock}, suggest that the addition of more accurate robot dynamics (e.g., slippage) and noise do not improve the situation. 
This implies that stronger robot capabilities are necessary to achieve the goal. In the future, we wish to explore the opposite direction, i.e., understanding the robot capabilities necessary to execute a given task.  Towards this end, we plan to explore the theory of knowledge in distributed systems~\cite{Moses16}, which, in our context, characterizes the information that should be known to the robots to execute a task. This may help determining both the type of sensors needed and the controller structure for a given task.

\section*{Acknowledgements}
The authors thank Idit Keidar, Yoram Moses, and Itai Panasoff for fruitful discussions, and to Avishav Engle for proofreading.

\bibliographystyle{IEEEtran}

\bibliography{bibtex/bib/refs.bib}

\begin{thebibliography}{10}
\providecommand{\url}[1]{#1}
\csname url@rmstyle\endcsname
\providecommand{\newblock}{\relax}
\providecommand{\bibinfo}[2]{#2}
\providecommand\BIBentrySTDinterwordspacing{\spaceskip=0pt\relax}
\providecommand\BIBentryALTinterwordstretchfactor{4}
\providecommand\BIBentryALTinterwordspacing{\spaceskip=\fontdimen2\font plus
\BIBentryALTinterwordstretchfactor\fontdimen3\font minus \fontdimen4\font\relax}
\providecommand\BIBforeignlanguage[2]{{%
\expandafter\ifx\csname l@#1\endcsname\relax
\typeout{** WARNING: IEEEtran.bst: No hyphenation pattern has been}%
\typeout{** loaded for the language `#1'. Using the pattern for}%
\typeout{** the default language instead.}%
\else
\language=\csname l@#1\endcsname
\fi
#2}}

\bibitem{halloy2007social}
J.~Halloy, G.~Sempo, G.~Caprari, C.~Rivault, M.~Asadpour, F.~Tâche, I.~Saïd, V.~Durier, S.~Canonge, J.~M. Amé, C.~Detrain, N.~Correll, A.~Martinoli, F.~Mondada, R.~Siegwart, and J.~L. Deneubourg, ``Social integration of robots into groups of cockroaches to control self-organized choices,'' \emph{Science}, vol. 318, no. 5853, pp. 1155--1158, 2007.

\bibitem{werfel2014designing}
J.~Werfel, K.~Petersen, and R.~Nagpal, ``Designing collective behavior in a termite-inspired robot construction team,'' \emph{Science}, vol. 343, no. 6172, pp. 754--758, 2014.

\bibitem{Blumenberg0B23}
P.~Blumenberg, A.~Schmidt, and A.~T. Becker, ``Computing motion plans for assembling particles with global control,'' in \emph{{IEEE/RSJ International Conference on Intelligent Robots and Systems}}, 2023, pp. 7296--7302.

\bibitem{trianni2003evolving}
V.~Trianni, R.~Gro{\ss}, T.~H. Labella, E.~{\c{S}}ahin, and M.~Dorigo, ``Evolving aggregation behaviors in a swarm of robots,'' in \emph{Advances in Artificial Life}.\hskip 1em plus 0.5em minus 0.4em\relax Springer, 2003, pp. 865--874.

\bibitem{barel2021probabilistic}
A.~Barel, T.~Dag{\`e}s, R.~Manor, and A.~M. Bruckstein, ``Probabilistic gathering of agents with simple sensors,'' \emph{SIAM Journal on Applied Mathematics}, vol.~81, no.~2, pp. 620--640, 2021.

\bibitem{dovrat2017gathering}
D.~Dovrat and A.~M. Bruckstein, ``On gathering and control of unicycle a (ge) nts with crude sensing capabilities,'' \emph{IEEE Intelligent Systems}, vol.~32, no.~6, pp. 40--46, 2017.

\bibitem{gazi2005swarm}
V.~Gazi, ``Swarm aggregations using artificial potentials and sliding-mode control,'' \emph{IEEE Transactions on Robotics}, vol.~21, no.~6, pp. 1208--1214, 2005.

\bibitem{cohen2005convergence}
R.~Cohen and D.~Peleg, ``Convergence properties of the gravitational algorithm in asynchronous robot systems,'' \emph{SIAM Journal on Computing}, vol.~34, no.~6, pp. 1516--1528, 2005.

\bibitem{oasa1997robust}
Y.~Oasa, I.~Suzuki, and M.~Yamashita, ``A robust distributed convergence algorithm for autonomous mobile robots,'' in \emph{IEEE International Conference on Systems, Man, and Cybernetics.}, vol.~1, 1997, pp. 287--292.

\bibitem{cieliebak2012distributed}
M.~Cieliebak, P.~Flocchini, G.~Prencipe, and N.~Santoro, ``Distributed computing by mobile robots: Gathering,'' \emph{SIAM Journal on Computing}, vol.~41, no.~4, pp. 829--879, 2012.

\bibitem{flocchini2019distributed}
P.~Flocchini, G.~Prencipe, N.~Santoro, \emph{et~al.}, ``Distributed computing by mobile entities,'' \emph{Current Research in Moving and Computing}, vol. 11340, no.~1, 2019.

\bibitem{gauci2014self}
M.~Gauci, J.~Chen, W.~Li, T.~J. Dodd, and R.~Gro{\ss}, ``Self-organized aggregation without computation,'' \emph{International Journal of Robotics Research}, vol.~33, no.~8, pp. 1145--1161, 2014.

\bibitem{mondada2009puck}
F.~Mondada, M.~Bonani, X.~Raemy, J.~Pugh, C.~Cianci, A.~Klaptocz, S.~Magnenat, J.-C. Zufferey, D.~Floreano, and A.~Martinoli, ``The e-puck, a robot designed for education in engineering,'' in \emph{Conference on Autonomous Robot Systems and Competitions}, vol.~1, no.~1, 2009, pp. 59--65.

\bibitem{daymude2021deadlock}
J.~J. Daymude, N.~C. Harasha, A.~W. Richa, and R.~Yiu, ``Deadlock and noise in self-organized aggregation without computation,'' in \emph{International Symposium on Stabilizing, Safety, and Security of Distributed Systems}.\hskip 1em plus 0.5em minus 0.4em\relax Springer, 2021, pp. 51--65.

\bibitem{klancar2017wheeled}
G.~Klancar, A.~Zdesar, S.~Blazic, and I.~Skrjanc, \emph{Wheeled mobile robotics: from fundamentals towards autonomous systems}.\hskip 1em plus 0.5em minus 0.4em\relax Butterworth-Heinemann, 2017.

\bibitem{Moses16}
Y.~Moses, ``Knowledge in distributed systems,'' in \emph{Encyclopedia of Algorithms}, 2016, pp. 1051--1055.

\end{thebibliography}

\end{document}